\documentclass[twoside]{article}

\usepackage[accepted]{aistats2022}
\usepackage{amssymb}                  
\usepackage{amsthm}                   
\usepackage{thmtools, thm-restate}
\usepackage{subcaption}               
\usepackage[hidelinks]{hyperref}      
\usepackage[capitalise]{cleveref}     
\usepackage{algorithm}                
\usepackage[noend]{algpseudocode}     
\usepackage{color}
\usepackage{dsfont}                   
\usepackage{listings}                 
\usepackage{tikz}
\usepackage{etoolbox}                 
\usepackage[makeroom]{cancel}         
\usepackage{todonotes}
\usepackage{balance}
\usepackage[export]{adjustbox}        
\usepackage{multirow}                 
\usepackage{bm}

\newcommand{\beq}[1]{\begin{equation} \eqlab{#1}}
\newcommand{\eeq}{\end{equation}}
\def\bal#1\eal{\begin{align}#1\end{align}}

\newcommand{\bsub}{\begin{subequations}}
\newcommand{\esub}{\end{subequations}}
\newcommand{\nn}{\nonumber}

\newcommand{\eqlab}[1]{\label{eq:#1}}

\renewcommand{\vec}[1]{\bm{#1}}

\newcommand{\mN}{\mathcal{N}}

\newcommand{\br}[1]{\left\lbrack #1 \right\rbrack}
\newcommand{\paren}[1]{\left(#1\right)}

\newcommand{\mr}[1]{\mathrm{#1}}
\newcommand{\vt}[1]{\left.#1\right\vert}

\newcommand{\yyy}{\vec{y}}

\newcommand{\zzz}{\vec{z}}

\DeclareMathOperator*{\argmin}{arg\,min}

\newcommand{\meanp}[2]{\mathbb{E}_{#1}\br{#2}}
\newcommand{\kl}[2]{\mr{KL}\paren{\vt{\vt{#1}}#2}}

\newenvironment{rcases}{\left.\begin{aligned}}{\end{aligned}\right\rbrace}

\newtheorem{theorem}{Theorem}
\newtheorem{lemma}[theorem]{Lemma}

\theoremstyle{definition} \newtheorem{definition}{Definition}
\theoremstyle{definition} \newtheorem{example}{Example}[section]

\newcommand{\circled}[2][]{%
  \tikz[baseline=(char.base)]{%
    \node[shape = circle, draw, inner sep = 1pt,scale=0.75]
    (char) {\phantom{\ifblank{#1}{#2}{#1}}};%
    \node at (char.center) {\makebox[0pt][c]{\scriptsize #2}};}}
\robustify{\circled}
\newcommand{\rom}[1]{%
  \textup{\uppercase\expandafter{\romannumeral#1}}%
}

\usetikzlibrary{decorations.pathreplacing,calc}
\newcommand{\tikzmark}[1]{\tikz[overlay,remember picture] \node (#1) {};}
\newcommand*{\AddNote}[4]{%
    \begin{tikzpicture}[overlay, remember picture]
        \draw [decoration={brace,amplitude=.75em},decorate,ultra thick,blue]
            ($(#3)!(#1.north)!($(#3)-(0,1)$)$) --  
            ($(#3)!(#2.south)!($(#3)-(0,1)$)$)
                node [align=center, text width=3cm, pos=0.5, anchor=west, xshift=0.75em] {#4};
    \end{tikzpicture}
}%

\newcommand{\algorithmicbreak}{\textbf{break}}

\DeclareFixedFont{\ttb}{T1}{txtt}{bx}{n}{8} 
\DeclareFixedFont{\ttm}{T1}{txtt}{m}{n}{8}  
\definecolor{deepblue}{rgb}{0,0,0.5}
\definecolor{deepred}{rgb}{0.6,0,0}
\definecolor{deepgreen}{rgb}{0,0.5,0}

\makeatletter
\lstdefinestyle{mystyle}{
    basicstyle=\ttm,
    language=Python,
    backgroundcolor=\color{white},
    keywordstyle=\ttb\color{deepgreen},
    emph={MyClass,__init__},
    emphstyle=\ttb\color{deepred},
    stringstyle=\color{deepblue},
    commentstyle=\color{red},
    frame=tb,
    showstringspaces=false,
    literate=
    {sample}{{{\color{deepblue}sample{}}}}5
    {observe}{{{\color{deepblue}observe{}}}}6
    {constraint}{{{\color{deepblue}constraint{}}}}9
    {rs_start}{{{\color{deepred}rs\_start{}}}}7
    {rs_end}{{{\color{deepred}rs\_end{}}}}5
    {rejection_sample}{{{\color{black}rejection\_sample{}}}}{14}
}
\makeatother
\usepackage[round]{natbib}

\begin{document}

\runningtitle{Amortized Rejection Sampling in Universal Probabilistic Programming}

\runningauthor{{\scriptsize Naderiparizi, \'Scibior, Munk, Ghadiri, Baydin, Gram-Hansen, Schroeder de Witt, Zinkov, Torr, Rainforth, Teh, Wood}}

\twocolumn[
\aistatstitle{Amortized Rejection Sampling in\\Universal Probabilistic Programming}
\aistatsauthor{Saeid Naderiparizi$^{1}$, Adam \'Scibior$^{1}$, Andreas Munk$^1$, Mehrdad Ghadiri$^2$\\{\bf At{\i}l{\i}m G\"{u}ne\c{s} Baydin$^3$, Bradley Gram-Hansen$^{3}$, Christian Schroeder de Witt$^{3}$}\\
{\bf Robert Zinkov$^{3}$, Philip Torr$^{3}$, Tom Rainforth$^{3}$, Yee Whye Teh$^{3}$, Frank Wood$^{1,4,5}$}}
\aistatsaddress{$^1$University of British Columbia, $^2$Georgia Institute of Technology, $^3$University of Oxford\\
$^4$MILA, $^5$CIFAR AI Chair} ]

\begin{abstract}
  Naive approaches to amortized inference in probabilistic programs with unbounded loops can produce estimators with infinite variance. This is particularly true of importance sampling inference in programs that explicitly include rejection sampling as part of the user-programmed generative procedure.  In this paper we develop a new and efficient amortized importance sampling estimator.  We prove finite variance of our estimator and empirically demonstrate our method's correctness and efficiency compared to existing alternatives on generative programs containing rejection sampling loops and discuss how to implement our method in a generic probabilistic programming framework.
\end{abstract}

\section{INTRODUCTION} \label{sec:intro}

It is now understood how to apply probabilistic programming inference techniques to generative models written in ``universal'' probabilistic programming languages (PPLs) \citep{van2018introduction}.  While the expressivity of such languages allows users to write generative procedures naturally, this flexibility introduces complexities, some of surprising and subtle character.  For instance there is nothing to stop users from using rejection sampling loops to specify all or part of their generative model, something that is quite natural to do and we have seen in practice.  While existing inference approaches may asymptotically produce correct inference results for such programs, the reality, which we discuss at length in this paper, is murkier.

The specific problem we address, that of efficient amortized importance-sampling-based inference in models with user-defined rejection sampling loops is more prevalent than it might seem on first consideration.  Our experience suggests that rejection sampling within generative model specification is actually the rule rather than the exception when programmers use universal languages for model specification.  To generate a single draw from anything more complex than standard distribution effectively requires either adding a new probabilistic primitive to the language (beyond most users), hard conditioning on constraint satisfaction (inefficient under most forms of universal PPL inference), or a user-programmed rejection loop.  A major example of this is sampling from a constrained distribution, like a truncated normal or a distribution constrained on a circle. If the model specification language does not have the primitive for the constrained distribution we want, the most natural way to generate such a variate is via user-programmed rejection.  More sophisticated examples abound in simulators used in the physical sciences \citep{baydin2018efficient, baydin2019etalumis}, chemistry \citep{cai2007exact,ramaswamy2010partial,slepoy2008constant}, and other domains \citep{stuhlmuller2014reasoning}. Implicit rejection sampling loops also exist in models containing simulators that guard against certain configurations and, in such cases, must restart after re-sampling the configurations \cite{warrington2019coping}.  Note that the issue we address here is not related to hard rejection via conditioning, i.e.,~\citet{ritchie2015controlling} and related work.  Ours is specifically about rejection sampling loops within the generative model program, whereas the latter is about developing inference engines that are reasonably efficient even when the user specified program has a constraint-like observation that produces an extremely peaked posterior.

The first inference algorithms for languages that allowed generative models containing rejection sampling loops to be written revolved around Markov chain Monte Carlo (MCMC) \citep{goodman2008a,wingate2011a} and sequential importance sampling (SIS) \citep{wood2014new} using the prior as the proposal distribution and then mean-field variational inference \citep{wingate2013automated}. Those methods were very inefficient, prompting extensions of PPLs providing programmable inference capabilities \citep{mansinghka2014venture,scibior2019formally}. Efforts to speed up inference since then have revolved around amortized inference \citep{gershman2014amortized}, where a slow initial off-line computation is traded against fast and accurate test-time inference.  Such methods work by training neural networks that  quickly map a dataset either to a variational posterior \citep{ritchie2016deep} or to a sequence of proposal distributions for SIS \citep{le2016inference}. This paper examines and builds on the latter ``Inference Compilation'' (IC) approach.

Unbounded loops potentially require integrating over infinitely many latent variables. With IC each of these variables has its own importance weight and the product of all the weights can have infinite variance, resulting in an importance sampler with unstable convergence. Furthermore, the associated self-normalizing importance sampler with any finite number of samples can converge to an arbitrary point, giving wrong inference results without warning. It is therefore necessary to take extra steps to ensure convergence of the importance sampler resulting from IC when unbounded loops are present. 

In this paper we present a solution to this problem for the common case of rejection sampling loops.  We establish, both theoretically and empirically, that computing importance weights naively in this situation can lead to arbitrarily bad posterior estimates. We develop a novel estimator to remedy this problem. A preview of the problem and the proposed solution is shown in \cref{fig:illustration}.
\begin{figure}[t]
    \begin{tabular}{|c|c|}
    \hline
    \begin{subfigure}{0.45\linewidth}
    \small
    \vspace{0.3cm}
    \input{algorithms/original.tex}
    \caption{Original program} \vspace{0.3cm} \label{alg:sample-program}
    \end{subfigure} &
    \begin{subfigure}{0.45\linewidth}
    \small
    \input{algorithms/inference.tex}
    \caption{Inference compilation} \label{alg:inference-program}
    \end{subfigure} \\ \hline
    \begin{subfigure}{0.45\linewidth}
    \small \vspace{0.3cm}
    \input{algorithms/collapsed.tex}
    \caption{Equivalent to above} \vspace{0.3cm} \label{alg:collapsed-program}
    \end{subfigure} &
    \begin{subfigure}{0.45\linewidth}
    \small
    \input{algorithms/collapsed-inference.tex}
    \caption{Our IS estimator} \label{alg:collapsed-inference}
    \end{subfigure} \\ \hline
    \end{tabular}
    \caption{(\subref{alg:sample-program}) Illustrates the problem we are addressing. Existing, naive approaches to inference compilation use trained proposals for the importance sampler with proposal $q$ shown in (\subref{alg:inference-program}), where $w$ can have infinite variance, even when each $w^k$ individually has finite variance, as $k$ is unbounded. There exists a simplified program (\subref{alg:collapsed-program}) equivalent to (\subref{alg:sample-program}) and ideally we would like to perform inference using the importance sampler in (\subref{alg:collapsed-inference}). While this is not directly possible, since we do not have access to the conditional densities required, our method approximates this algorithm, without introducing infinite variance.
    } \label{fig:illustration}
\end{figure}

\section{PROBLEM FORMULATION} \label{sec:problem-formulation}
Our formulation of the problem will be presented concretely starting from the example probabilistic program shown in Figure \ref{alg:sample-program}. Even though both the problem and our solution are more general, applying to all probabilistic programs with rejection sampling loops regardless of how complicated they are, this simple example captures all the important aspects of the problem.
As a reminder, inference compilation refers to offline training of importance sampling proposal distributions for all random variables in a program \citep{le2016inference}.   Existing, naive approaches to inference compilation for the program in Figure~\ref{alg:sample-program} correspond to the importance sampler shown in Figure~\ref{alg:inference-program} where there is some proposal learned for every random choice in the program.  While the weighted samples produced by this method result in unbiased estimates, the variance of the weights can be very high and potentially infinite due to the unbounded number of $w^k$s.  To show this, we start by more precisely defining the meaning of the sampler in \cref{alg:inference-program}.

\newcommand\wn{w_{\mathrm{IC}}}

\begin{definition}[Naive weighing] \label{def:program}
  Let $p(x,y,z)$ be a probability density such that all conditional densities exist.
  For each $y$, let $q(x,z|y)$ be a probability density such that $p(x,z|y)$ is absolutely continuous with respect to it. Let $c(x,z)$ be a Boolean condition, and $A$ be the event that $c$ is satisfied such that $p(A|x,y) \geq \epsilon$ and $q(A|x, y) \geq \epsilon$ for all $(x,y)$ and some $\epsilon > 0$.
  Let $x \sim q(x|y)$ and let $z^k \overset{i.i.d.}{\sim} q(z|x,y)$ and $w^k = \frac{p(z^k|x)}{q(z^k|x,y)}$ for all $k \in \mathbb{N}^+$. Let $L = \min\{k | c(x,z^k)\}$, $z = z^L$ and ${\wn = \frac{p(x)}{q(x|y)} p(y|x,z) \prod_{k=1}^L w^k}$.
\end{definition}

For simplicity, in \cref{def:program} and the rest of this paper we let $A$ be the event that $c$ is satisfied. We assert that \cref{def:program} corresponds to the program in \cref{alg:inference-program}. A rigorous correspondence could be established using formal semantics methods, such as the construction of \citet{scibior2017denotational}, but this is beyond the scope of this paper. Although, as we prove later in \cref{thm:equivalence}, the resulting importance sampler correctly targets the posterior distribution, the variance of $\wn$ is a problem, and it is this specific problem that we tackle in this paper.

Intuitively, a large number of rejections in the loop leads to a large number of $w^k$ being included in $\wn$ and the variance of their product tends to grow quickly. In the worst case, this variance may be infinite, even when each $w^k$ has finite variance individually. This happens when the proposed samples are rejected too often, which is formalized in the following theorem.

\begin{restatable}{theorem}{infvariance}
    \label{thm:infinity-variance}
    Under assumptions of \cref{def:program}, if the following condition holds with positive probability under $x \sim q(x|y)$
    \begin{equation}\label{eq:infinity-variance-condition}
        \meanp{z \sim q(z | x,y)}{\frac{p(z | x)^2}{q(z | x,y)^2} (1 - p(A|x,z))} \geq 1
    \end{equation}
    then the variance of $\wn$ is infinite.
\end{restatable}
A proof of this theorem is in \cref{app:proofs:variance}.

\cref{thm:infinity-variance} means that importance sampling with proposals other than the prior may hurt more than help in the case of rejection sampling loops and there is no trivial way to ensure \cref{eq:infinity-variance-condition} does not hold or to detect if it holds for a particular proposal. 
Furthermore, under the conditions of \cref{thm:infinity-variance}, existing IC schemes are effectively useless in practice, even though they are still unbiased in principle.
Since the central limit theorem governs convergence of IS estimators \citep{geweke1989bayesian}, the convergence rates fail when the variance of weights is infinite. Consequently, it leads to a slow to converge and unstable estimator that may exhibit strong biases with any finite number of samples \citep{robert1999monte,koopman2009testing}.
Worse still, even when the variance is finite but large, it may render the effective sample size too low for practical applications, a phenomenon we have observed repeatedly in practice.
What remains is to derive an alternative way to compute $\wn$ that guarantees avoiding such problems arising from rejection sampling loops and in practice leads to larger effective sample sizes than existing methods.

\section{APPROACH} \label{sec:approach}

A starting point to the presentation of our algorithm is to observe that the program in \cref{alg:sample-program} is equivalent to \cref{alg:collapsed-program}, where $z$ satisfying the condition $c$ is sampled directly. \cref{alg:collapsed-inference} presents an importance sampler targeting \ref{alg:collapsed-program}, obtained by sampling $z$ directly from $q$ under the condition $c$. Note that the sampling processes in \ref{alg:inference-program} and \ref{alg:collapsed-inference} are the same, only the weights are computed differently. We now provide a definition for the weights in \ref{alg:collapsed-inference}.

\newcommand\wc{w_{\mathrm{C}}}

\begin{definition}[Collapsed weighing] \label{def:idealized}
Extending \cref{def:program}, let
\begin{equation}
    \label{eq:weight}
    \wc = \frac{p(x)}{q(x|y)} \underbrace{\frac{p(z|x, A)}{q(z|x, A, y)}}_{\circled{T}} p(y|x, z).
\end{equation}
\end{definition}

Note that $\meanp{}{\wn} = \meanp{}{\wc}$, as we prove in \cref{thm:equivalence}. However, since $\wc$ only involves a fixed number (three) of terms, we can expect it to avoid the aforementioned problems with exploding variance. Unfortunately, we can not directly compute $\wc$.

In \cref{eq:weight} we can directly evaluate all terms except $\circled{T}$, since, $p(z|x, A)$ and $q(z|x, A, y)$ are  defined implicitly by the rejection sampling loop. Applying Bayes' rule to this term gives the following equality:
    \begin{equation} \label{eq:bayes}
        \circled{T} =
        \underbrace{\frac{q(A|x, y)}{p(A|x)}}_{\circled{1}}
        \underbrace{\frac{p(z|x)}{q(z|x, y)}}_{\circled{2}}
        \underbrace{\frac{\cancel{p(A|z, x)}}{\cancel{q(A|z, x, y)}}}_{\circled{3}}.
    \end{equation}

The term $\circled{2}$ can be directly evaluated, since we have access to both conditional densities and the term $\circled{3}$ is always equal to $1$, since $A$ only depends on $x$ and $z$, and is determined by $c(z, x)$ which is common between $p$ and $q$. However, the term $\circled{1}$, which is the ratio of acceptance probabilities under $q$ and $p$, can not be computed directly and we need to estimate it. We provide separate unbiased estimators for $q(A|x, y)$ and $\frac{1}{p(A|x)}$.

For $q(A|x, y)$ we use straightforward Monte Carlo estimation of the following expectation:
\begin{align}
    q(A|x, y) &= \int q(A|z,x, y) q(z|x, y) dz\nn\\
    &= \int c(z,x) q(z|x, y) dz\nn\\
    &= \meanp{z \sim q(z|x, y)}{c(z, x)}\label{eq:zq_defintion}
\end{align}

For $\frac{1}{p(A|x)}$ we use Monte Carlo estimation after applying the following standard lemma:
\begin{lemma}
\label{lemma:geometric-dist}
Let $A$ be an event that occurs in a trial with probability $p$. The expectation of the number of trials to the first occurrence of $A$ is equal to $\frac{1}{p}$.
\end{lemma}

It is important that these estimators are constructed independently of $z$ being sampled to ensure that we obtain an unbiased estimator for $\wc$ specified in \cref{eq:weight}. Also, it is important to note these two values $q(A|x,y)$ and $\frac{1}{p(A|x)}$ are state-dependent, they are not only specific to a rejection sampling loop, but also depend on the state of the program when entering the loop.  We put together all these elements to obtain our final method in \cref{alg:our-algorithm}. More formally, the weight obtained by our method is defined as follows.

\newcommand\wo{w_{\mathrm{ARS}}}

\begin{definition}[Our weighting] \label{def:ours}
Extending \cref{def:program}, let
$z_i' \overset{i.i.d.}{\sim} q(z|x,y)$ for $i \in 1,\dots,N$ and $K$ be the number of $z_i'$ for which
$c(x,z_i')$ holds.
Let $\zzz''_{j} = (z_{j,1}'', \dots, z_{j,T_j}'')$ be sequences of potentially
varying length for $j \in 1, \dots, M$ with $z_{j,l}'' \overset{i.i.d.}{\sim} p(z|x)$ such that for all
$j$, $T_j$ is the smallest index $l$ for which $c(x,z_{j,l}'')$ holds.  Let $ T =
\frac{1}{M}\sum_{j=1}^M T_j$. Finally, let
\begin{equation}
    \wo = \frac{p(x)}{q(x|y)}\frac{KT}{N}\frac{p(z|x)}{q(z|x,y)}p(y|x,z) .
\end{equation}
\end{definition}

Throughout this section we have only informally argued that the three importance samplers presented target the same distribution. With all the definitions in place we can make this argument precise in the following theorem.

\begin{theorem} \label{thm:equivalence}
For any $N \geq 1$ and $M \geq 1$, and all values of $(x,y,z)$,
\begin{equation}
 \meanp{}{\wn | x,y,z} = \wc = \meanp{}{\wo | x,y,z} .
\end{equation}
\end{theorem}
\begin{proof}
For the second equality, use \cref{eq:zq_defintion}, then \cref{lemma:geometric-dist}, \cref{eq:bayes}, and finally \cref{eq:weight}.
\begin{align}
&\meanp{}{\wo | x,y,z} \\
= &\frac{p(x)}{q(x|y)}\frac{p(z|x)}{q(z|x,y)}p(y|x,z) \frac{1}{N} \meanp{z'}{K} \meanp{z''}{T} \\
= &\frac{p(x)}{q(x|y)}\frac{p(z|x)}{q(z|x,y)}p(y|x,z)q(A|x, y)\frac{1}{p(A|x)} \\
= &\frac{p(x)}{q(x|y)} \frac{p(z|x, A)}{q(z|x, A, y)} p(y|x, z)
= \wc
\end{align}
For the first equality, use \cref{eq:app:proofs:variance:weight-mean} in \cref{app:proofs:variance} to get
\begin{align}
&\meanp{}{\wn | x,y,z} \\
= &\frac{p(x)}{q(x|y)} p(y|x,z) \, w^L \, \meanp{z^{1:L-1}}{\prod_{k=1}^{L-1} w^k} \\
= &\frac{p(x)}{q(x|y)} p(y|x,z) \, \frac{p(z|x)}{q(z|x, y)} \, \frac{q(A|x,y)}{p(A|x)} = \wc
\end{align}
\end{proof}

\begin{algorithm}[t]
    \begin{algorithmic}[1]
        \caption{Pseudocode for our algorithm applied to the probabilistic program from \cref{alg:sample-program}.}
        \label{alg:our-algorithm}
        \State $x \sim q(x | y)$
        \State $w \leftarrow \frac{p(x)}{q(x | y)}$
        \For{$k \in \mathbb{N^+}$}
        \State $z^k \sim q(z|x,y)$ 
        \If { $c(x, z^k)$ }
            \State $z = z^k$
            \State \algorithmicbreak 
        \EndIf
        \EndFor
        \State $w \leftarrow w\, \frac{p(z|x)}{q(z|x,y)}$ \;
        \vspace{.5em}
        \State $K \leftarrow 0$ \tikzmark{zq-top}
        \For{$i \in 1, \ldots N$}
            \State $z'_i \leftarrow q(z|x, y)$
            \State $K \leftarrow K + c(z, x)$ \tikzmark{zq-right}
        \EndFor \tikzmark{zq-bottom}
        \vspace{.5em}
        \For{$j \in 1, \ldots M$} \tikzmark{zp-top}
            \For{$l \in \mathbb{N^+}$}
                \State $z''_{j,l} \leftarrow q(z|x, y)$
                \If { $c(x, z''_{j,l})$ } \tikzmark{zp-right}
                    \State $T_j \leftarrow l$
                    \State \algorithmicbreak
                \EndIf
            \EndFor
        \EndFor
        \State $T \leftarrow \frac{1}{M} \sum_{j=1}^M T_j$ \tikzmark{zp-bottom}
        \vspace{.5em}
        \State $w \leftarrow w\,\frac{KT}{N}$
        \State $w \leftarrow w\, p(y|z,x)$
    \end{algorithmic}
    \AddNote{zq-top}{zq-bottom}{zq-right}{Estimate $q(A|x, y)$ using \cref{eq:zq_defintion}}
    \AddNote{zp-top}{zp-bottom}{zp-right}{Estimate $\frac{1}{p(A|x)}$ using \cref{lemma:geometric-dist}}
    \end{algorithm}

Since all three importance samplers use the same proposal distributions for $(x,z)$, \cref{thm:equivalence} shows that they all target the same distribution, which is the posterior distribution specified by the original probabilistic program in \cref{alg:sample-program}.

Finally, we can prove that our method handles inference in rejection sampling loops without introducing infinite variance. Note that variance may still be infinite for reasons not having to do with the rejection sampling loop, if $q(x|y)$ and $q(z|x,y)$ are poorly chosen.


\begin{theorem}
If $\wc$ from \cref{def:idealized} has finite variance, then $\wo$ from \cref{def:ours} has finite variance, for any $N \geq 1$ and $M \geq 1$.
\end{theorem}
\begin{proof}
Note that conditionally on $(x,y,z)$, $K$ follows a binomial distribution. Therefore, $\mathrm{Var}[\frac{K}{N} | x,y,z] < 1 < \infty$. Then, note that conditionally on $(x,y,z)$, $T_j$s are independent of each other and follow a geometric distribution. Therefore,
\begin{align*}
    &\mathrm{Var}[T_j | x,y,z] = \frac{1 - p(A | x)}{p(A|x)^2} < \frac{1}{p(A|x)^2}\\
    \Rightarrow &\mathrm{Var}[T | x,y,z] < \frac{1}{p(A|x)^2} < \frac{1}{\epsilon^2} < \infty.
\end{align*}
Also, conditionally on $(x,y,z)$, $\frac{K}{N}$
and $T$ are independent, so $\mathrm{Var}[\frac{KT}{N} | x, y, z] < B$ for some constant $B < \infty$.
Then see that ${\wo = \wc \frac{p(A|x)}{q(A|x,y)} \frac{KT}{N}}$. Finally, using the law of total variance, we get
\begin{align}
&\mathrm{Var}[\wo] = \meanp{}{\mathrm{Var}[\wo|x,y,z]}\nn\\
&\hspace{5.5em}+ \mathrm{Var}[\meanp{}{\wo|x,y,z}] \\
&= \meanp{}{\left(\wc \frac{p(A|x)}{q(A|x,y)}\right)^2 \mathrm{Var}\left[\frac{KT}{N}\bigg|x,y,z\right]}\nn\\
&\quad + \mathrm{Var}[\wc] \\
&\leq \meanp{}{\wc^2 \frac{1}{\epsilon^2} B} + \mathrm{Var}[\wc] < \infty
\end{align}
\end{proof}
\paragraph{Training proposals} Our method does not concern training the IC network. It is mainly about computing sample weights once the network is trained. However, an important assumption made in our method is the same proposal $q(z|x, y)$ is used in all iterations of a rejection sampling loop. Therefore, it should be reflected in the training process in IC as well. We follow the approach proposed in \cite{baydin2018efficient} which discards the rejected samples at training time and only uses the sample values that conclude the loop for training. As a result, the training samples for $z$ are from $p(z|x,A)$, which means, the distribution of training data obtained from the programs in \cref{alg:sample-program,alg:collapsed-program} are identical. We provide more details and a discussion on training proposals in \cref{app:training}.

\section{EXPERIMENTS} \label{sec:experiments}
We illustrate our method by performing inference in three example probabilistic programs that include rejection sampling loops. In each program the latent variables are identified via \lstinline{sample} statements. The inferred posterior distribution is conditioned on observed values, identified via \lstinline{observe} statements. Two of our experiments are designed so that ground truth inference results can be derived analytically.

\begin{figure*}[tbp]
  \centering
  \includegraphics[trim=0 114 0 0,clip]{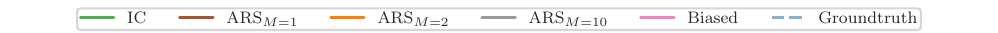}
  \includegraphics[trim=0 7 0 15,clip]{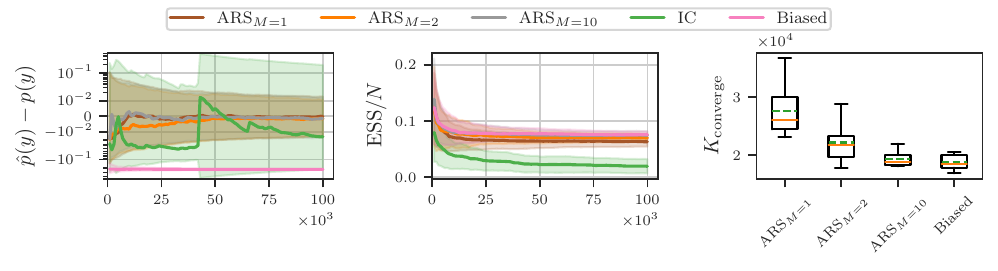}
  \includegraphics[trim=0 5 0 15,clip]{figs/exp_mini_sherpa_a_812.pdf}
  \caption{(Top) results of Marsaglia and (bottom) Mini-SHERPA experiments.
  In both experiments we estimate marginal likelihood of an observation.
  Left plots show how different methods converge to the ground truth marginal likelihood. Middle plots show the (normalized) ESS for each method. Right plots are box plots of the required number of samples to ensure convergence. In these box plots green dashed and orange solid lines show the mean and median, respectively.
  Our method with any $M$ converges to ground truth with lower variance compared to IC. As expected, larger $M$ leads to faster convergence, lower variance, and higher ESS.
  See \cref{fig:app-more-results:results-table} in \cref{app:more-results} for a summary of final values reached in these plots.
  }
  \label{fig:results-marsaglia-mini_sherpa}
\end{figure*}

\newcommand{\kconverge}{K_\text{converge}}

We evaluate the efficacy of our approach in several ways including (I) computing the effective sample size (ESS) of IS estimators, (II) empirically comparing the convergence rates of different methods to ground truth values, and (III) reporting the number of samples required to ensure convergence of the IS estimators. We adopt a convergence test proposed by \citet{chatterjee2018sample} and report the sample size required for IS estimators, $\kconverge$. Formally, define
\begin{equation}
    Q_K := \frac{\max_{1 \leq k \leq K} w^{k}}{\sum_{k=1}^K w^{k}},
\end{equation}
where $w^k$ is defined as in \cref{def:program}. According to this convergence test, an IS estimate with $K$ samples is converged if $\meanp{}{Q_K} < \epsilon$, for a predefined $\epsilon$. Finally, $\kconverge$ is the smallest value for which the IS estimator is converged.

The specific methods we compare are:
\begin{itemize}
  \item{Naive (IC)}: Like the approach in \cref{alg:inference-program} this uses a proposal for all iterations of rejection sampling loops. Final importance weights are computed by multiplying all the weights for all samples in the trace, accepted and rejected.
  \item{Amortized Rejection Sampling (ARS$_{M=m}$)}:
  This is our method (\cref{alg:our-algorithm}) with ${M=m}$ and ${N=\max(m, 10)}$ as defined in \cref{def:ours}, not multiplying in the weights of any rejected samples on the trace.
  \item{Ablated (Biased)}: Implements the incorrect variant of ARS that does not estimate nor multiple in the correction factor $\frac{q(A|x,y)}{p(A|x)}$. This is only shown to see the effect of correction factors in terms of convergence speed and estimation bias.
\end{itemize}

Note that the variance and convergence of our method depends primarily on the probabilities of rejection in the model and proposal, and not on other features of the rejection sampler. Therefore, our experimental results do not change significantly for more complex programs.

\subsection{Marsaglia}
Marsaglia polar method \citep{marsaglia1964convenient} is a pseudo-random number sampling method for generating samples from a Normal distribution. It samples a point $(a, b)$ uniformly from a unit circle and applies change of variables $x_1 = a \sqrt{\frac{-2 \log(a^2+b^2)}{a^2+b^2}}$ and $x_2 = b \sqrt{\frac{-2 \log(a^2+b^2)}{a^2+b^2}}$. Then, $x_1$ and $x_2$ are two independent samples distributed as $\mN(0, 1)$. The most straight-forward way of sampling from a circle is by sampling from a square and rejecting the sample if it lies outside the circle. This is a common use case of rejection sampling, to sample from a constrained probability space.

We implement a Gaussian with unknown mean model with two observations $y_1$ and $y_2$. The generative process is defined as ${\mu\sim \mN(\mu_0,\sigma_{0}^2)}$, ${y_i|\mu \sim \mN(\mu,\sigma^2)}$
where sampling from $\mN(\mu_0, \sigma_{0}^2)$ is implemented by the Marsaglia polar method. Program~\ref{code:marsaglia} in \cref{app:experiment-details-marsaglia} shows the implementation of this model.

We train an LSTM-based inference compilation network on this model to learn proposals and use them in a SIS engine to estimate the marginal likelihood $p(y_1, y_2)$ for different observations.
Top row of \cref{fig:results-marsaglia-mini_sherpa} (left) shows the estimation error between the marginal likelihood $p(y_1, y_2)$ and its SIS estimation $\hat{p}(y_1, y_2) = \sum_{k=1}^{K} w^k$. The figure shows an aggregation of 100 runs of the experiment. Estimates by our method always converge to the true marginal likelihood with any $M>0$. Further, larger $M$ leads to faster convergence and higher ESS. On the other hand, the figure shows that IC fails to converge after 100,000 draws. It is worth mentioning that not for all observations IC fails to converge and/or has low ESS. The main problem is whether this happens is not identifiable in practice.

\subsection{Mini-SHERPA}
Mini-SHERPA is an event generator of a simplified model of high-energy reactions of particles. Its naming comes from SHERPA \citep{gleisberg2009event}, the state-of-the-art simulator of high-energy reactions of particles. SHERPA has up to thousands of latent variables and widely uses rejection sampling loops which is a major source of difficulty in inference in this model \citep{baydin2018efficient,baydin2019etalumis}. Mini-SHERPA, however, has up to 11 latent variables (excluding variables rejected in rejection sampling loops) and up to two rejection sampling loops. In its simulation process, depicted in \cref{fig:sherpa-process-picture}, it samples a 3-dimensional momentum for a starting particle, samples a type of decay (known as ``decay channel''), then samples momentum of each of the resulting particles. At the end of the simulation, it produces a noisy measurement of the energy deposited by the resulting particles on the 2-dimensional surface of a detector as the observation. An example observation is shown on the right side of \cref{fig:sherpa-process-picture}. We provide more details including a pseudocode of the simulator in \cref{app:experiment-details-marsaglia}.

We train an LSTM-based inference compilation network on this model to learn proposals. We then use the learned proposals in a SIS engine to estimate the marginal likelihood of a given observation $p(y)$. The ground truth value of marginal likelihood is estimated by importance sampling with prior as proposal for the variables inside rejection sampling loops. Since sampling from prior is suboptimal, we draw more than 12 million samples to estimate the ground truth value. The results of this experiment are shown in bottom row of \cref{fig:results-marsaglia-mini_sherpa}. Similar to the previous experiment, ARS with any $M > 0$ converges to the ground truth value while larger values of $M$ generally lead to higher ESS and faster convergence. IC performs poorly in this experiment and fails to converge after 100,000 samples. We provide additional results with other observations in the appendix.

\begin{figure}[tbp]
    \centering
    \includegraphics[width=0.5\columnwidth,valign=c]{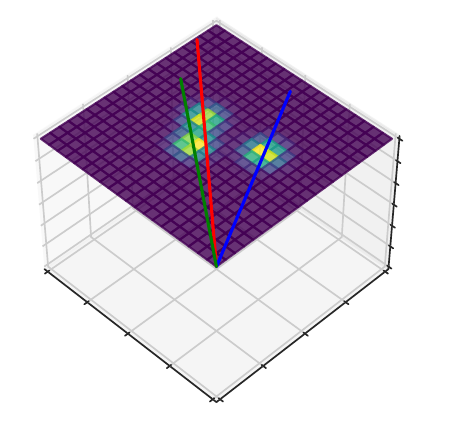}
    \includegraphics[width=0.4\columnwidth,trim=190 0 0 0, clip,valign=c]{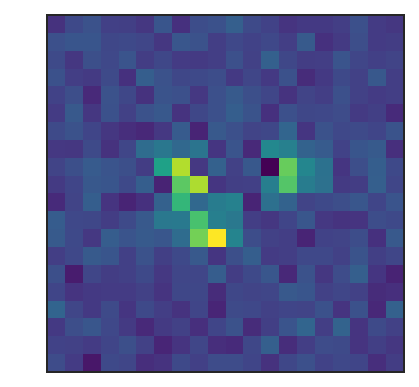}
    \caption{An example of the simulation process in the Mini-SHERPA experiment. (Left) shows its simulation process schematically. It shows decay of a starting particle into three new particles. Each of the red, blue and green lines show the trajectory of a resulting particle. The 2D grid on top shows the measured energy deposited by the particles. (Right) shows the observation in this model which is a noisy version of the energy grid.}
    \label{fig:sherpa-process-picture}
\end{figure}

In all experiments, as expected, the Biased method does not converge to the correct value, but it usually has high ESS. This is potentially misleading as the samples do not have the extra weight variance introduced by Monte Carlo estimate of the correction factors for rejection sampling loops.

\subsection{Beta-Bernoulli}
In our last experiment, we focus on comparing ARS with a baseline proposed by \citet{baydin2018efficient}, which we term ``Prior''. This baseline uses the prior as proposal for the variables within rejection sampling loops, ignoring the learned proposal. This approach sidesteps the need for correction factors because they are simply equal to 1.

Intuitively, because ``Prior'' does not involve the additional Monte-Carlo estimation that ARS introduces, it should have lower variance. On the other hand, if the prior is far from the true posterior, the estimator will have high variance. Hence, depending on the model and observations, ``Prior'' might be better or worse than ARS. To investigate this in practice, we implement a Beta-Bernoulli experiment,
$$x \sim \text{Beta}(\alpha, \beta)$$
$$y_i \sim \text{Bernoulli}(x) \text{ for } 1 \leq i \leq n.$$
However, we explicitly implement sampling from the Beta prior by repeatedly sampling from a Uniform distribution and accepting based on the ratio of the Beta and Uniform distributions. Program~\ref{code:beta-bern} shows an implementation of the model. Further details, such as the exact form of the acceptance function, is provided in the appendix.

\begin{figure}[tbp]
    \centering
    \includegraphics[width=\columnwidth]{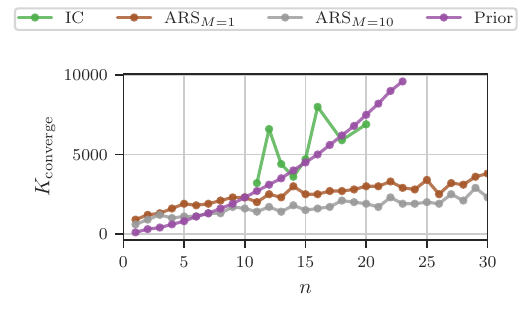}
    \caption{Results of the Beta-Bernoulli experiment. In this plot, $K_\text{converge}$ is the required number of samples to ensure convergence and $n$ is the model's parameter. As $n$ grows, the difference between the prior and the true posterior increases, deteriorating performance of the ``Prior'' approach. Our method, even with a very limited budget to estimate the correction factors, quickly outperforms ``Prior''. The missing points of the lines means fails to converge after 10,000 samples.}
    \label{fig:results_beta}
\end{figure}

\begin{figure}[bt]
\begin{lstlisting}[caption=Beta-Bernoulli, frame=tb, label=code:beta-bern]
while True:
  rs_start()
  x = sample(Uniform(0, 1))
  u = sample(Uniform(0, 1))
  if c(x,u):
      rs_end()
      break
for i in range(n):
  observe(Bernoulli(x), y[i])
\end{lstlisting}
\end{figure}

This model is parametrized by the prior parameters $\alpha, \beta$ and the number of observations $n$. We control the rejection rate by changing $\alpha, \beta$ and the closeness of the prior and posterior by changing $n$. In our experiments $y_i = \text{True}$ for all $i$.

In this experiment, we do not train the proposals. Instead, since the posterior distributions are tractable, we analytically derive the distributions that optimize inference compilation's objective (that is, the posterior distributions in the ``collapsed'' Beta-Bernoulli program) and manually choose the proposals accordingly. See the appendix for more details.

In \cref{fig:results_beta} we report $\kconverge$, the smallest value of $K$ for which the estimator passes the convergence test. This plot demonstrates that depending on the program and the observations we solve inference for, the ``Prior'' approach may perform better or worse than ARS. Importantly where the true posterior is far from the prior (when $n$ is large), the ``Prior'' approach quickly fails to converge in less than 10,000 draws. Note that ARS$_{M=1}$ converges quickly even with a very limited budget to estimate the correction factors. ARS converges faster with larger values of $M$ while IC fails to converge most of the time.

\section{IMPLEMENTATION} \label{sec:imp}
In the previous sections we have presented and experimentally validated our method on a few programs. Our method applies much more broadly, in fact to all probabilistic programs with arbitrary stochastic control flow structures and any number of rejection sampling loops. This includes nesting such loops to arbitrary degree.  The only constraint is that observations can not be placed inside the loops, i.e., \emph{no conditioning inside rejection sampling loops}. We conjecture that our method produces correct weights for all probabilistic programs satisfying this constraint, but proving that is beyond the scope of this paper and would require employing sophisticated machinery for constructing formal semantics, such as developed by \citet{scibior2017denotational}. Nonetheless, we provide a brief proof sketch in \cref{app:proofs:more-complex-programs}.

To enable practitioners to use our method in its full generality we have implemented it in PyProb\footnote{\url{https://github.com/pyprob/pyprob}} \citep{le2016inference}, a universal probabilistic programming library written in Python. The particulars of such a general purpose implementation pose several difficult but interesting problems, including identifying rejection sampling loops in the original program, addressing particular rejection sampling loops, and engineering solutions that allow acceptance probabilities bespoke to each loop to be estimated by repeatedly executing the loops with different proposal distributions. In this section we describe some of these challenges and discuss our initial approach to solving them.
\begin{figure}[tbp]
    \begin{minipage}{.49\linewidth}
    \begin{lstlisting}[caption=Original, frame=tb]
x = sample(P_x)
while True:

    z = sample(P_z(x))
    if c(x, z):

        break
observe(P_y(x,z), y)
return x, z
    \end{lstlisting}
    \end{minipage}\hfill
    \begin{minipage}{.49\linewidth}
    \begin{lstlisting}[caption=Annotated, frame=tb]
x = sample(P_x)
while True:
    rs_start()
    z = sample(P_z(x))
    if c(x, z):
        rs_end()
        break
observe(P_y(x,z), y)
return x, z
    \end{lstlisting}
    \end{minipage}
    \caption{An illustration of annotations required by our system. To apply our method we only require that entry and exit to each rejection sampling loop be annotated with \lstinline{rs_start} and \lstinline{rs_end} respectively. Our implementation then automatically handles the whole inference process, even if the annotated loops are nested.} \label{fig:sample-program}
\end{figure}

The first challenge is identifying rejection sampling loops in the probabilistic program itself.  One mechanism for doing this is to introduce a rejection sampling primitive, macro, or syntactic sugar into the probabilistic programming language itself whose arguments are two functions: the acceptance function and the body of the rejection sampling loop.  While this may be feasible, our approach lies in a different part of the design space, given our choice to implement in PyProb and its applicability to performing inference in existing stochastic simulators.  In this setting there are two other design choices: some kind of static analyzer that automates the labelling of rejection sampling loops by looking for rejection sampling motifs in the program (unclear how to accomplish this in a reliable and general way) or providing the probabilistic programmer functions that need to be carefully inserted into the existing probabilistic program to demarcate where rejection sampling loops start and end.

We chose the latter approach in this paper.  In the programs in Section~\ref{sec:experiments} we silently introduced the functions \lstinline[columns=fixed]{rs_start} and \lstinline[columns=fixed]{rs_end} to tag the beginning and end of rejection sampling loops.  These functions are used to inform an inference engine about the scope of each rejection sampling loop, in particular so that all \lstinline[columns=fixed]{sample} statements in between calls to \lstinline[columns=fixed]{rs_start} and \lstinline[columns=fixed]{rs_end} can be tracked.  \cref{fig:sample-program} illustrates where these primitives have to be inserted in probabilistic programs with rejection loops to invoke our ARS techniques.

The specific implementation details are infeasible to cover thoroughly and requires substantial review of PyProb internals. However, the key functionality enabled by these tags includes two critical things:
\textbf{(\rom{1})} we need to be able to execute additional iterations of every rejection sampling loop in the program to compute our estimator of $\frac{q(A|x,y)}{p(A|x)}$.  For every  \lstinline[columns=fixed]{rs_start} we have to be able to continue the program multiple times, executing the rejection loop both proposing from $p$ and $q$, terminating the continuation once the matching \lstinline{rs_end} is reached. Efficient implementations of this make use of the same forking ideas that made ``probabilistic C'' possible \citep{paige2014compilation}. \textbf{(\rom{2})} We have to be able to identify rejection sampler start and end pairs and
design a special addressing scheme for the samples in rejection sampling loops such that the rejected samples are replaced by the latter accepted ones.

Our PyProb implementation\footnote{\url{https://github.com/plai-group/amortized-rejection-sampling}} addresses all of these issues. We provide more details on our implementation in \cref{app:implementation-details}.

\section{DISCUSSION} \label{sec:discussion}
We have addressed an issue in amortized importance-sampling-based inference for universal probabilistic programming languages.
We have demonstrated that even simple rejection sampling loops can cause major problems for existing probabilistic programming inference algorithms.
Particularly, we showed empirically and theoretically that SIS can perform poorly in presence of rejection sampling loops, even in simple models with only a few rejection sampling loops.
Our proposed method is an unbiased estimator, often with lower variance than naive SIS estimators.

Although our method has a new source of variance, that of additional Monte Carlo estimators which can make its variance higher than naive SIS in some cases, we have proved it is guaranteed to always only add a finite variance to the estimates. The absence of this guarantee is a major shortcoming of naive SIS methods, as there is no easy way of predicting if they will have infinite variance. As a result, incorrect estimates can be made without warning. Therefore, they cannot be used safely when a program that contains rejection sampling loops.

The cost of taking our approach is somewhat subtle but involves needing to estimate the exit probabilities of all rejections sampling loops in the program under both the prior and the proposal. At inference time, once a rejection sampling loop is encountered, the inference engine must ``pause'' and estimate them on the fly. This can be slow and increase the implementation complexity of the probabilistic programming system. However, using forking and multiprocessing capabilities avoids the time overhead of our method as the correction factor estimates are made in parallel with the main inference engine.

However, this fix of amortized importance-sampling-based inference for universal probabilistic programming systems is significant, particularly as it pertains to uptake of this kind of probabilistic programming system.  Currently, users of such systems who use rejection sampling in their generative models may experience the probabilistic programming system as confusingly not working.  This will be due to potentially non-convergent non-diagnosable behavior we elaborated on, which in turn leads to poor sample efficiency.  Our work makes it so that efficient, amortized inference engines work for probabilistic programs that users actually write and in so doing removes a major impediment to the uptake of universal probabilistic programming systems in general.    
\balance

\subsubsection*{Acknowledgements}
We thank Lukas Heinrich for providing the implementation of the Mini-SHERPA simulator and helping with its experiment in this paper. We acknowledge the support of the Natural Sciences and Engineering Research Council of Canada (NSERC), the Canada CIFAR AI Chairs Program, and the Intel Parallel Computing Centers program. Additional support was provided by UBC's Composites Research Network (CRN), and Data Science Institute (DSI). Bradley Gram-Hansen is supported by the EPSRC Autonomous Intelligent Systems and Machines grant. Christian Schroeder de Witt is generously supported by the Cooperative AI Foundation. This research was enabled in part by technical support and computational resources provided by WestGrid (www.westgrid.ca), Compute Canada (www.computecanada.ca), and Advanced Research Computing at the University of British Columbia (arc.ubc.ca).

\bibliography{bib}


\clearpage
\appendix

\thispagestyle{empty}

\onecolumn \makesupplementtitle

\section{PROOFS}\label{app:proofs}
\subsection{Proof of \cref{thm:infinity-variance}}\label{app:proofs:variance}
\infvariance*
\begin{proof}
    In this proof, we carry the assumptions and definitions from \cref{def:program}, with the exception that we use subscript for denoting weights and samples in iterations of the loop i.e., $z^k \rightarrow z_k$ and $w^k \rightarrow w_k$.
    
    We first compute the variance of the rejected sample weights. As a reminder, $L$ is a random variable denoting the number of iterations until acceptance, hence all the iterations until $L-1$ were rejected while the $L^{\text{th}}$ iteration is accepted. Let $\meanp{q}{\prod_{k=1}^{L-1} w_k}$ denote the mean value of the product of all the weights corresponding to the rejected samples when sampling from the proposal $q$, and define $\meanp{q}{\prod_{k=1}^{L-1} w^{2}_k}$ similarly. As another reminder, $A$ stands for the event of the condition $c$ being satisfied (acceptance in an iteration of the rejection sampling loop) and $\overline{A}$ stands for the event of $c$ not being satisfied. As stated before, we are computing the variance of the rejected sample weights which is ${\meanp{q}{\prod_{k=1}^{L-1} w^{2}_k} - \meanp{q}{\prod_{k=1}^{L-1} w_k}}$. We start with the second term.
    \begin{align}
        \meanp{q}{\prod_{k=1}^{L-1} w_k}
        &= \sum_{l=1}^{\infty} \mathbb{P}[L=l] \prod_{k=1}^{l-1}\meanp{z_k \sim q(z|x, y)}{w_k|\overline{A}}
        = \sum_{l=1}^{\infty} \mathbb{P}[L=l] \prod_{k=1}^{l-1}\meanp{z \sim q(z|x, y)}{w|\overline{A}}\label{eq:rejected-weights-mean}\\
        &= \sum_{l=1}^{\infty} q(A|x,y) q(\overline{A}|x,y)^{l-1} \prod_{k=1}^{l-1}\meanp{z \sim q(z|x, y)}{w|\overline{A}}\label{eq:rejected-weights-mean-2}
    \end{align}
    The first equality in \cref{eq:rejected-weights-mean} comes from the fact that $L=l$ means the rejection sampling loop was iterated $l$ times to get the first accepted sample and it implies that for all $1 \leq k \leq l-1$, $z_k$ is rejected. Hence, the inner expectation is conditioned on $\overline{A}$. Moreover, since the $z_k$ samples are independent given $x$ and $y$, the expectation commutes with the product. The second equality comes from the independence of $z_k$ samples too.

    Now we expand the last term in \cref{eq:rejected-weights-mean-2} i.e., $\meanp{z \sim q(z|x, y)}{w | \overline{A}}$,
    \begin{equation}
        \meanp{z \sim q(z|x, y)}{w | \overline{A}}
        = \meanp{z \sim q(z|x, y)}{\frac{p(z|x)}{q(z|x, y)} \frac{1 - c(x,z)}{q(\overline{A}|x, y)}}
        = \frac{1}{q(\overline{A}|x,y)} \meanp{z \sim p(z|x)}{1 - c(x, z)}
        = \frac{p(\overline{A}|x)}{q(\overline{A}|x, y)} \label{eq:rejected-weights-mean-step2}
    \end{equation}
    In the first equality in \cref{eq:rejected-weights-mean-step2}, $w = \frac{p(z|x)}{q(z|x, y)}$ and the other term accounts for conditioning on $\overline{A}$. In the second equality, we have assumed $q(z|x, y)$ is a valid importance sampling proposal for $p(z|x)$ i.e., if $\mathcal{Z}$ is the space of possible values for $z$,\footnote{Although in all of our experiments this assumption holds, it might not be true depending on the training scheme. Refer to \cref{app:training:ic-perfect} for a more detailed discussion.}
    $$\forall z \in \mathcal{Z}:\; p(z|x) > 0 \Rightarrow q(z|x, y) > 0$$
    Therefore, substituting Equation \ref{eq:rejected-weights-mean-step2} into \ref{eq:rejected-weights-mean-2},
    \begin{equation}
        \meanp{q}{\prod_{k=1}^{L-1} w_k} = q(A|x,y) \sum_{k=1}^{\infty} p(\overline{A}|x)^{k-1} = \boxed{\frac{q(A|x,y)}{p(A|x)}}\label{eq:app:proofs:variance:weight-mean}
    \end{equation}
    Next, we compute the expected value of squared of weights. It can be derived similar to \cref{eq:rejected-weights-mean,eq:rejected-weights-mean-2}
    \begin{equation}
        \meanp{q}{\prod_{k=1}^{L-1} w_{k}^2} = \sum_{l=1}^{\infty} q(A|x,y) q(\overline{A}|x,y)^{l-1} \prod_{k=1}^{l-1} \meanp{z \sim q(z|x,y)}{w^2|\overline{A}}\label{eq:rejected-weights-squared-mean}
    \end{equation}
    We now expand the last term in \cref{eq:rejected-weights-squared-mean} and, similar to \cref{eq:rejected-weights-mean-step2}, get the following,
    \begin{equation}
        \meanp{z \sim q(z|x,y)}{w^2|\overline{A}}
        = \meanp{z \sim q(z|x,y)}{\frac{p(z|x)^2}{q(z|x,y)^2} \frac{1-c(x,z)}{q(\overline{A}|x,y)}}
        = \frac{1}{q(\overline{A}|x,y)}\meanp{z \sim q(z|x,y)}{\frac{p(z|x)^2}{q(z|x,y)^2} p(\overline{A}|x, z)}
    \end{equation}
    For notational simplicity, define $S_{p, q} = \meanp{z \sim q(z|x,y)}{\frac{p(z|x)^2}{q(z|x,y)^2} p(\overline{A}|x, z)}$. Hence,
    \begin{equation}
        \meanp{q}{\prod_{k=1}^{L-1} w_{k}^2} = q(A|x,y) \sum_{l = 1}^{\infty} (S_{p, q})^{l-1}\label{eq:weight-squared-mean}
    \end{equation}
    Finally, according to \cref{eq:weight-squared-mean,eq:app:proofs:variance:weight-mean} the variance of the rejected sample weights is equal to
    \begin{equation}
        {\meanp{q}{\prod_{k=1}^{L-1} w^{2}_k} - \meanp{q}{\prod_{k=1}^{L-1} w_k}}
        = q(A|x,y) \sum_{l = 1}^{\infty} (S_{p, q})^{l-1} - \frac{q(A|x,y)}{p(A|x)} \label{eq:rejected-weight-variance}
    \end{equation}
    
    Since the second term in $\frac{p(A|x,y)}{p(A|x)}$ and in \cref{eq:rejected-weight-variance} is finite, if $q(A|x, y)\sum_{l = 1}^{\infty} (S_{p, q})^{l-1}$ is not finite, the variance of the rejected weights will be infinite. Additionally, since $S_{p,q}$ is independent of $q(A|x,y)$, it reduces to finiteness of $\sum_{l = 1}^{\infty} (S_{p, q})^{l-1}$.
    \begin{equation}
        \text{if } S_{p,q} \geq 1 \Rightarrow \sum_{l = 1}^{\infty} (S_{p, q})^{l -1} \text{ is infinite} \Rightarrow Var\left(\prod_{i=1}^{k-1} w_{i}\right) \text{ is infinite}
    \end{equation}
    Noting that if the variance of the weights of a subset of samples (the rejected samples in this case) is infinite, the variance of the whole weights would be infinite, completes the proof.
\end{proof}
Note that even though our proof was presented on the example program in \cref{alg:sample-program}, it is not limited to it. In general, in a program that contains multiple (even nested) rejection sampling loops, if the inequality in \cref{eq:infinity-variance-condition} holds for any of its rejection sampling loops, it makes the variance of the IC weights infinite.

\subsection{Proof sketch of ARS providing correct weights for more complex programs}\label{app:proofs:more-complex-programs}
There are many ways to do it, but one would be to follow denotational semantics \citep{scibior2017denotational}, which identifies probabilistic programs with functions from inputs to unnormalized measures on the outputs. This semantics is compositional, so showing that two programs denote the same measure implies full contextual equivalence. All four programs in \cref{fig:illustration} then define the same measure on $(x, z)$ for all $y$. Finally, by \cref{thm:equivalence} the program defined in \cref{alg:our-algorithm} denotes the same measure as well, so can be substituted for the original program in all contexts.

\section{A DISCUSSION ON TRAINING PROPOSALS}\label{app:training}
\subsection{Naive Weighting with Perfectly Trained Proposals}\label{app:training:ic-perfect}
In this part we investigate the validity of the existing IC weighting method (\cref{def:program}) under perfectly trained proposals.

As stated in the main text, in \cite{baydin2018efficient} the proposals are trained using only the samples that conclude the rejection loop, hence, the training data drawn from the original program (\cref{alg:sample-program}) has the same distribution as the training data from the collapsed program (\cref{alg:collapsed-program}). Therefore, the perfect proposal is the same for these two programs.

The proposals are trained by minimizing the expected forward KL between the posterior and the proposal,
\begin{align}
    q^* &= \argmin_{q}\meanp{p(y)}{\kl{p(x, z|y)}{q(x, z|y)}}. \label{eq:ic-training-objective}
\end{align}
Therefore, the perfect proposal $q^*$ matches the posterior of the collapsed program $q^*(x, z|y) = p(x, z|y)\; \forall y \in \mathcal{Y}$, where $\mathcal{Y}$ is the space of all observations that can be generated from the model. More formally, $\mathcal{Y}$ is the space of all $y$ such that ${p(y) = \int\int p(x,z,y) dz dx > 0}$.

Given an observation ${y\in \mathcal{Y}}$, define ${\gamma(x, z) = \gamma(x)\gamma(z|x)}$ to be the posterior in the collapsed program. Hence, ${q^*(x|y) = \gamma(x)}$ and ${q^*(z|x,y) = \gamma(z|x)}$, in both the collapsed and original programs. Now we can investigate if the condition in \cref{app:proofs:variance} holds with the perfect proposal:
\begin{equation}
    \meanp{z \sim q^*(z|x, y)}{\frac{p(z|x)^2}{q^*(z|x, y)^2} (1 - p(A|z,x))}
    = 
    \meanp{z \sim \gamma(z|x)}{\frac{p(z|x)^2}{\gamma(z|x)^2} (1 - p(A|z,x))}
\end{equation}
Since $\gamma(z|x)$ is the true posterior for the collapsed program, every sample drawn from it $z \sim \gamma(z|x)$ satisfies $c(x, z)$ therefore,
\begin{equation}
    p(A|z,x)=1,\; \forall z \sim \gamma(z|x)
    \Rightarrow \meanp{z \sim \gamma(z|x, y)}{\frac{p(z|x)^2}{\gamma(z|x, y)^2} (1 - p(A|z,x))} = 0 < 1
\end{equation}
So, if the proposals are trained perfectly, we would not have the problem of infinite variance. However, note that a perfectly trained proposal using only the accepted samples does not necessarily provide a valid importance sampling proposal for the distribution $p(z|x)$ i.e., $q(z|x, y)$ can have zero mass on parts of the space where $p(z|x) > 0$. Consequently, it can lead to biased estimates. This issue is illustrated in the following simple example.

\begin{example}
    Consider the following original and collapsed programs with the same format as Figure~\ref{alg:sample-program}.
    \begin{figure}[h]
        \begin{minipage}{.49\linewidth}
        \begin{lstlisting}[caption=Original, frame=tb, stepnumber=1]
x = sample(Uniform(low=1, high=2))
while True:
    rs_start()
    z = sample(Normal(mean=0, std=1))
    if z < 0:
        rs_end()
        break
observe(Normal(mean=z, std=x), y)
return x, z
        \end{lstlisting}
        \end{minipage}\hfill
        \begin{minipage}{.49\linewidth}
        \begin{lstlisting}[caption=Collapsed, frame=tb]
x = sample(Uniform(low=1, high=2))


z = sample(TruncatedNormal(mean=0, std=1, max=0))



observe(Normal(mean=z, std=x), y)
return x, z
        \end{lstlisting}
        \end{minipage}
        \end{figure}

Both of these programs implement the following generative model:
\begin{align*}
    &x \sim \text{Uniform}(1, 2)\\
    &z \sim \mN_{(-\infty, 0)}(0, 1)\\
    &y \sim \mN(z, x)
\end{align*}
Where $\mN_{(-\infty, 0)}$ is a truncated normal distribution with a maximum value of zero and unbounded minimum.

Since these programs are equivalent, inferences on them should have identical results\footnote{As long as it does not involve inference about the rejected samples}. However, under the perfect proposal assumption, although the proposal samples have the same distribution, the weights are different. For example, consider a sample $x^k \sim q^*(x|y) = \gamma(x)$ and $z^k \sim q^*(z|x^k, y) = \gamma(z|x^k)$, if we refer to the weights for the original and collapsed programs as $w_{IC}$ and $w_C$,
\begin{align}
    w_{IC} &= \frac{U(x^k; 0, 1)}{\gamma(x)} \frac{\mN(z^k; 0, 1)}{\gamma(z^k|x^k)}\\
    w_{C} &= \frac{U(x^k; 0, 1)}{\gamma(x)} \frac{\mN_{(-\infty, 0)}(z^k; 0, 1)}{\gamma(z^k|x^k)} \label{eq:perfect-proposal-example-w-c}
\end{align}
In \cref{eq:perfect-proposal-example-w-c}, $z^k \sim \gamma(z|x^k)$, $z^k < 0$, therefore, $\mN_{(-\infty, 0)}(z^k; 0, 1) = 2\mN(z^k; 0,1)$. Hence, $w_{C} = 2w_{IC}$ while in order to get the same result, importance sampling weights should be equal.
\end{example}

It is worth mentioning that in this simple program because the computed weights in this program differ by a multiplicative constant, the normalized weights would be equal (in case of self-normalized importance sampling). However, if the difference is state-dependent (i.e., the distribution of $z$ depends on the value of $x$ sampled at the previous step for example $z \sim \mN(x, 1)$ instead), self-normalization would not help.

It is important to note that in our experiments this problem of proposals having zero mass on the rejected subspace does not happen. In the particular inference network we used, all the proposals are guaranteed to have a support broader than or the same as the prior. It is made possible by choosing the parametrized family of proposal distributions from the ones that have broader support than the priors. However, because they are trained on the accepted samples only, once trained, they can put arbitrarily small mass on the rejected subspace. This in turn makes the quantity in \cref{app:proofs:variance} larger and eventually, makes the variance infinite.

\subsection{An Alternative Training Scheme}\label{app:training:alternative}
Following \cref{app:training:ic-perfect} where we showed inference with naive IC weighting (Definition~\ref{def:program}) can be biased when the proposals are trained only on the accepted samples, in this section we provide another training method that, at its optimal point, provides proposals under which estimates with IC weighting are guaranteed to be unbiased and finite variance. We argue that if the proposals are trained optimally on both the accepted and the rejected samples, the resulting IC estimator has finite variance.

\begin{theorem}
    Consider an inference compilation network trained on both accepted and rejected samples of a rejection sampling loop. If it is trained optimally, \cref{eq:infinity-variance-condition} does not hold for it.
\end{theorem}
\begin{proof}
    In this case, the perfectly trained proposal is:
    \begin{equation}
        q^*(z|x, y) = p(\overline{A}|x) p(z|x, \overline{A}) + p(A|x) \gamma(z|x)
    \end{equation}
    Where following the notation in \cref{app:training:ic-perfect}, $\gamma(z|x)$ is the posterior $p(z|x,y)$. In this situation,
    \begin{itemize}
        \item If $z \sim q^*(z|x, y)$ is in $A$ (i.e., accepted), $p(z|x, \overline{A}) = 0$. Hence, in this region, $q^*(z|x, y) = p(A|x) \gamma(z|x)$
        \item If $z \sim q^*(z|x, y)$ is in $\overline{A}$ (i.e., rejected), $\gamma(z|x) = 0$. Hence, in this region, $q^*(z|x, y) = p(\overline{A}|x) p(z|x, \overline{A})$
    \end{itemize}
    We abuse the notation and reuse $A$ and $\overline{A}$ to denote the sub-spaces of accepted and rejected $z$, given the previous sample $x$ and drop the dependence on $x$ for notational simplicity. We split the inequality in \cref{eq:infinity-variance-condition} to these two accepted and rejected sub-spaces
    \begin{align}
        \meanp{z \sim q^*(z|x, y)}{\frac{p(z|x)^2}{q^*(z|x, y)^2} (1 - p(A|z,x))}
        &= \int_{z \in A \cup \overline{A}}{\frac{p(z|x)^2}{q^*(z|x, y)^2} (1 - p(A|z,x))} q^*(z|x,y) dz\nn\\
        &= I_A(p, q^*) + I_{\overline{A}}(p,q^*)
    \end{align}
    Where $I_A(p, q^*)$ and $I_{\overline{A}}(p, q^*)$ denote the value of the integral on the accepted and rejected subspaces, respectively. They can be simplified by replacing $q^*(z|x,y)$ by their value on those regions, as stated above,
    \begin{equation}
        I_A(p, q^*) = \int_{z \in A}\frac{p(z|x)^2}{\textcolor{blue}{p(A|x)^2 \gamma(z|x)^2}} \underbrace{(1 - p(A|z,x))}_{0} p(A|x) \gamma(z|x) dx = 0
    \end{equation}
    \begin{align}
        I_{\overline{A}}(p, q^*)
        &= \int_{z \in \overline{A}} \frac{p(z|x)^2}{\textcolor{blue}{p(\overline{A}|x)^2 p(z|x, \overline{A})^2}} \underbrace{(1 - p(A|z,x))}_{1} p(\overline{A}|x) p(z|x, \overline{A}) dx\\
        &= \int_{z \in \overline{A}} \frac{p(z|x)^2}{p(\overline{A}|x)^2 p(z|x, \overline{A})^2} \underbrace{p(\overline{A}|x) p(z|x, \overline{A})}_{\neq 0} dx\\
        &= \int_{z \in \overline{A}} \frac{p(z|x)^2}{\textcolor{blue}{p(\overline{A}|x) p(z|x, \overline{A})}} dx
        = \int_{z \in \overline{A}} \frac{p(z|x)^2}{\underbrace{\textcolor{blue}{p(\overline{A}|x, z)}}_{1} \textcolor{blue}{p(z|x)}} dx\\
        &= \int_{z \in \overline{A}} p(z|x) dx \leq 1 \label{eq:I-rej-subspace}
    \end{align}
    If in \cref{eq:I-rej-subspace} equality holds, it means that $A = \emptyset$, hence, $p(A|x) = 0$, equivalently, the rejection sampling loop never terminates which is an invalid rejection sampler. Therefore, ${\meanp{z \sim q^*(z|x, y)}{\frac{p(z|x)^2}{q^*(z|x, y)^2} (1 - p(A|z,x))} < 1}$ and the variance is finite.
\end{proof}

However, since $q^*$ is a convex combination of $p(z|x, \overline{A})$ and $\gamma(z|x)$, depending on how efficient the rejection sampling loop is implemented by the user (i.e., how small $p(A|x)$ is) $q^*$ can be arbitrarily far from the correct posterior $\gamma(z|x)$. It is important to generate samples close to $\gamma(z|x)$ because it will generate high-likelihood samples. Additionally, such a proposal can be wasteful computationally; it can have high rejection rate and it depends on the user code and out of the control of the inference engine.
\par With all that said, we proved finite variance only at the optimal point i.e., if the inference network is imperfect, naive application of sequential importance sampling can still lead to an estimator with infinite variance. Therefore, even with this new training scheme, naive sequential importance sampling cannot be safely used with guaranteed consistency guarantee.

\section{IMPLEMENTATION DETAILS}\label{app:implementation-details}
In this section we provide more details our implementation of ARS. In particular, we informally explain (I) our addressing scheme for random variables in rejection loops, (II) training the network such that it is only trained on the accepted samples, (III) running the inference network at test time.

\subsection{Random variable addresses}
Following the notation of \citet{le2016inference}, an \emph{execution trace} of a probabilistic program is a sequence
\begin{equation}
    (x_t, a_t, i_t)_{t=1}^{T},
\end{equation}
where $x_t$, $a_t$, and $a_t$ are respectively the sampled value, address, and instance value (call number) of the $t^{\mathrm{th}}$ random variable in a given trace.
An address $a_t$ is a unique identifier automatically generated for each \lstinline{sample} statement in the program. An instance value $i_t$ is a counter of how many times a \lstinline{sample} statement is executed in the same program trace i.e., $i_t = \sum_{j=1}^t \mathds{1}(a_t = a_j)$.
The inference network then learns proposal distributions $q_{a, i}$ corresponding to the addresses $a$ for all \lstinline{sample} statements in the program and their instance values $i$. Therefore, each proposal distribution is identified by $(a, i)$. As explained in the paper, we train the same proposal distribution for different iterations of a rejection loop. Consequently, the addressing scheme should be modified to reflect this requirement.

\newcommand{\rstack}{\mathbf{s}}
In our implementation, we first extend the definition of addresses and instance values to cover \lstinline{rs_start} statements as well. We then define a trace for rejection sampling loops (denoted by RS trace) as a sequence of $(x_t, a_t, i_t)$, similar to the program traces, but $x_t$ denotes the iteration of the loop here. We then modify the definition of a program trace to a sequence
\begin{equation}
    (x_t, r_t, a_t, i_t),
\end{equation}
$r_t$ is the identifier (address and instance value) of the latest \emph{active} rejection sampling loop, or $\emptyset$ if no active rejection sampling loop exists at time $t$. An active rejection sampling loop is a loop that is started but is not concluded yet. In other words, \lstinline{rs_start} has been executed, but its corresponding \lstinline{rs_end} has not been reached yet.

To handle rejection sampling loops, we maintain a stack $\rstack_t$ of all active rejection sampling loops. At the beginning of program execution, $\rstack_0 = \emptyset$. Every time the program enters a \emph{new} rejection sampling loop, we push its identifier to $\rstack$ and pop from its top when executing \lstinline{rs_end}. When running \lstinline{rs_start} at time $t$, in order to detect if it is the start of new rejection sampling loop or retrying the last one, we compare the address the \lstinline{rs_start} statement $a_t$ with the top of the stack.
\begin{itemize}
    \item If the addresses do not match, we identify this as a new rejection sampling loop and push its identifier $(a_t, i_t)$ to the top of the stack: $\rstack_t = \rstack_{t-1} \oplus (a_t, i_t)$, where $\oplus$ denotes pushing to the top of the stack. 
    \item If the addresses match, it is a retry of the latest loop. Let $(x, a, i)$ be the last item in RS stack (we know $a_t = a$). We first increase $x$ in RS stack by one, denoting a new iteration of the loop. Then discard every item in program trace that has a rejection sampling identifier matching $(a, i)$ i.e., $\{(x_j, r_j, a_j, i_j): j < t, r_j = (a, i)\}$. Then we set $\rstack_t = \rstack_{t-1}$ and the program execution continues.
\end{itemize}
With this definition, $r_t = \emptyset$ if $\rstack_{t}$ is empty. Otherwise, it is the top element of $\rstack_{t}$.

Therefore, once we execute a program, its program traces only contain accepted samples since all the rejected samples are instantly discarded and removed from the trace. Moreover, since each execution of \lstinline{rs_start} for a new rejection sampling loop gets a unique identifier, it can uniquely identify all rejection sampling loops, even for complicated program structures such as nested loops.

\subsection{Training}
With our modified definition of program traces, we can use the traces after their execution is finished as training data for inference compilation network. Since the rejected samples are discarded, the traces are identical to the ones sampled from the ``collapsed'' program.

\subsection{Inference}
One of the most important points to consider when performing inference is to ensure proposal distributions for different iterations of the same rejection sampling loop are the same. To satisfy this constraint, we should recover the LSTM network's state when retrying a rejection sampling loop. To accommodate it, we store the LSTM's state at the time of starting a new rejection sampling loop in the RS trace. In other words, we modify the definition of RS trace to be a sequence of $(x_t, a_t, i_t, h_t)$ where $h_t$ is the LSTM state. Then, we can simply restore LSTM's state once a rejection sampling loop retry is detected.

\section{DETAILS OF EXPERIMENTS}\label{app:experiment-details}
Our experiments are implemented using PyProb. The architecture of our inference compilation network is the LSTM-based network introduced in \citep{le2016inference} and existing in PyProb. The proposal distributions for Normal distributions are mixtures of 10 Normals and for Uniform distributions are mixtures of 10 Beta distributions with the same support as the prior. All the training is done by Adam optimizer \citep{kingma2014adam}. To compute $\kconverge$, we choose $\epsilon = 0.01$, unless otherwise specified.

\subsection{Marsaglia}\label{app:experiment-details-marsaglia}
    \paragraph{Training} The inference compilation network trained in this experiment has a single layer, 512 dimensional LSTM and is trained on 2 million random draws from the model with a learning rate of $10^{-3}$ and batch size of $512$. The specific distribution parameters in our model is $\mu_0 = 0$, $\sigma_0 = 1$, $\sigma = 0.1$.

    \paragraph{Model} Implementation of this model (including rejection sampling tags) is shown in Program~\ref{code:marsaglia}. Observations used in the experiments are $y_1 = y_2 = 0$ in \cref{fig:results-marsaglia-mini_sherpa}~(top) and \cref{fig:app-more-results:marsaglia}~(top) and $y_1 = y_2 = -1$ in \cref{fig:app-more-results:marsaglia} (bottom).

    \begin{minipage}[tbp]{\columnwidth}
        \begin{lstlisting}[caption=The Marsaglia experiment, label=code:marsaglia]
def GUM(mu_0, sigma_0, sigma, y1, y2):
    while True:
        rs_start()
        x1 = sample(Uniform(-1, 1))
        x2 = sample(Uniform(-1, 1))
        s = x2**2 + x2**2
        if s < 1:
            rs_end()
            break
    mu = x * sqrt(-2*log(s)/s)
    observe(Normal(mean=mu, std=sigma), y1)
    observe(Normal(mean=mu, std=sigma), y2)
    return mu
        \end{lstlisting}
    \end{minipage}

\subsection{Beta-Bernoulli}
    \subsubsection{Acceptance Function}
    The acceptance function \texttt{c(x, u)} is chosen in a way that a sample $x$ drawn from the ``base distribution'' gets accepted with probability $\Big(x(1 + \frac{\beta}{\alpha})\Big)^{\alpha - 1} \Big((1-x) (1 + \frac{\alpha}{\beta})\Big)^{\beta - 1}$, therefore,
    $$c(x, u) := \mathds{1}_{u \in C(x)} \text{ where } C(x): \left\{z \in [0, 1]: z \leq \Big(x(1 + \frac{\beta}{\alpha})\Big)^{\alpha - 1} \Big((1-x) (1 + \frac{\alpha}{\beta})\Big)^{\beta - 1}\right\}.$$
    Note that $c(x,u)$ depends on the parameters $\alpha$ and $\beta$, but this dependence in suppressed in the notation for simplicity. In our experiment we assume $\alpha=\beta$. It simplifies $C(x)$ to
    $$C(x): \left\{z \in [0, 1]: z \leq \Big( 4 x(1-x) \Big)^{\alpha - 1}\right\}.$$
    
    \begin{figure}[b]
        \begin{minipage}{.49\linewidth}
        \begin{lstlisting}[caption=Original, frame=tb, stepnumber=1, label=code:app:beta-bernoulli-orig]
while True:
    rs_start()
    x = sample(Uniform(0, 1))
    u = sample(Uniform(0, 1))
    if c(x, u):
        rs_end()
        break
for i in range(n):
    observe(Bernoulli(x), y[i])
        \end{lstlisting}
        \end{minipage}\hfill
        \begin{minipage}{.49\linewidth}
        \begin{lstlisting}[caption=Collapsed, frame=tb, label=code:app:beta-bernoulli-collapsed]
x = sample(Beta(alpha, beta))






for i in range(n):
    observe(Bernoulli(x), y[i])
        \end{lstlisting}
        \end{minipage}
        \caption{Probabilistic programs implementing the Beta-Bernoulli model.} \label{fig:beta-bernoulli-programs}
        \end{figure}

    \subsubsection{Proposal Distributions}
    We mentioned in the main paper that the proposals are chosen manually. Here we explain in more detail how the proposals are chosen.
    
    \paragraph{True posterior for $x$} Consider the probabilistic program as shown in \cref{fig:beta-bernoulli-programs}. We show the program in both ``original'' and ``collapsed'' versions (with the terminology from \cref{fig:illustration}.) We know the true posterior for the latent variable $x$ in Program~\ref{code:app:beta-bernoulli-collapsed} is $\text{Beta}(\alpha + n, \beta)$, since all the observations are ``True''.

    \paragraph{True posterior for $u$} Considering Program~\ref{code:app:beta-bernoulli-orig}, the true posterior of $u$ given the set of observations $\yyy = \{y_i\}_{i=1}^{n}$ and the previously sampled latent variable $x$ is $p(u | \yyy, x) = \text{Uniform}\left(0, \Big( 4 x(1-x) \Big)^{\alpha - 1} \right)$.

    Having these true posterior distributions in mind, we choose (nearly) perfect proposal distributions in our experiment, as summarized in \cref{tab:app:beta-bern-proposals}.

        \begin{table}[t]
            \centering
            \begin{tabular}{ |c||c|c| } 
             \hline
             Parameters & $q(x|\yyy)$ & $q(u|\yyy, x)$ \\
             \hline
             $\alpha = \beta = 2$ & $\text{Beta}(\alpha + n, \beta)$ & \text{Uniform}(0, 1) \\ 
             $\alpha = \beta = 10$ & $\text{Beta}(\alpha + n, \beta)$ & $\text{Mixture}_{0.99, 0.01} \left[ \text{Uniform}\left(0, \Big( 4 x(1-x) \Big)^{\alpha - 1} \right), \text{Uniform} (0, 1) \right]$ \\
             \hline
            \end{tabular}
            \caption{Manually chosen proposal distributions in the Beta-Bernoulli experiment. The first row corresponds to \cref{fig:results_beta} in the main text and \cref{fig:app-more-results:beta} (left) in \cref{app:more-results}. The second row corresponds to the additional results plot in \cref{fig:app-more-results:beta} (right). In this table, $\text{Mixture}_{0.99, 0.01}$ means a mixture of two distributions with $0.99$ and $0.01$ being the probabilities of each of mixtures respectively. This Mixture distribution is necessary to ensure $p$ is absolutely continuous with respect to $q$.}
            \label{tab:app:beta-bern-proposals}
        \end{table}

\subsection{Mini-SHERPA}
    In this experiment, the inference compilation has a 3-layer layer LSTM with dimension of 512. The network is trained on 2 million random draws from the model with a learning rate of $10^{-3}$ and a batch size of 512.

    The observation in this experiment is a noisy measurement of the energy dispatched by the simulated particles. Figure~\ref{fig:app:more-results:mini-sherpa-process} (bottom) shows the observations used in our Mini-SHERPA experiment. Each observation is a $20 \times 20$ image. The observation noise is a zero-mean independent multivariate Gaussian on image pixels.

    The implementation of this experiment, excluding the unnecessary simulation details is shown in Program~\ref{code:mini-sherpa}.

    \begin{figure}[btp]
        \centering
        \includegraphics[width=0.27\columnwidth,valign=c]{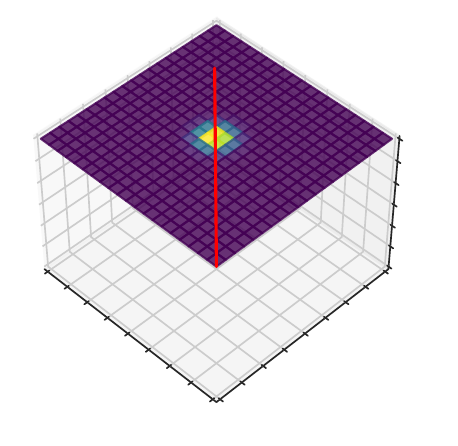}\hfill
        \includegraphics[width=0.27\columnwidth,valign=c]{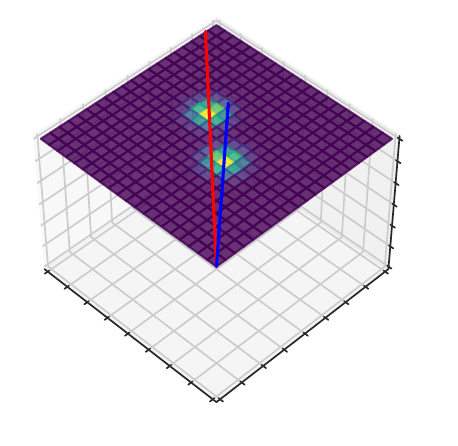}\hfill
        \includegraphics[width=0.27\columnwidth,valign=c]{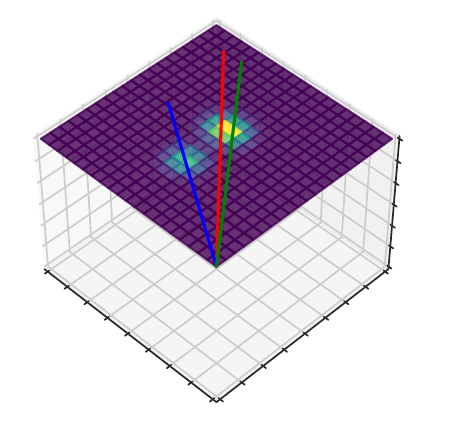}\hfill
        \includegraphics[width=0.27\columnwidth,trim=190 0 0 0, clip,valign=c]{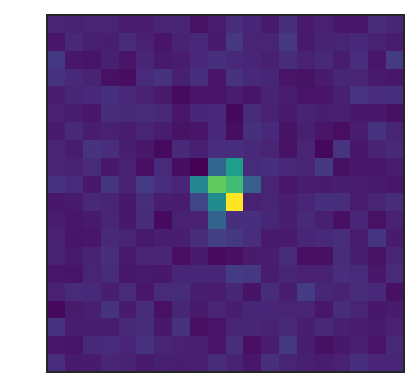}\hfill
        \includegraphics[width=0.27\columnwidth,trim=190 0 0 0, clip,valign=c]{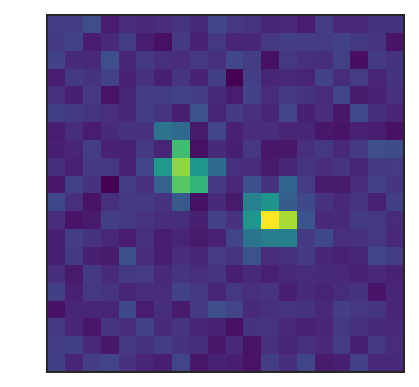}\hfill
        \includegraphics[width=0.27\columnwidth,trim=190 0 0 0, clip,valign=c]{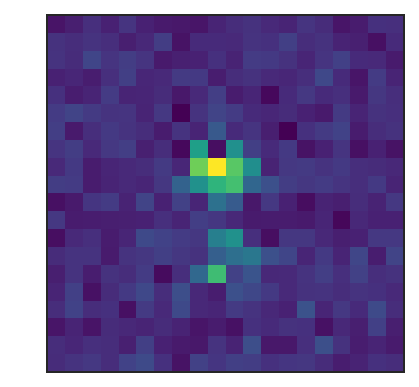}
        \caption{Simulation process (top row) and noisy observations (bottom row) in our Mini-SHERPA experiment. Rows correspond to events of channel 1, 2, and 3, respectively from left to right. These are the observations used in this paper. \cref{fig:results-marsaglia-mini_sherpa} (bottom row) corresponds to the left column and \cref{fig:app:more-results:mini-sherpa-results} corresponds to the next two columns.}
        \label{fig:app:more-results:mini-sherpa-process}
    \end{figure}

    \begin{figure}[htbp]
        \centering
        \includegraphics{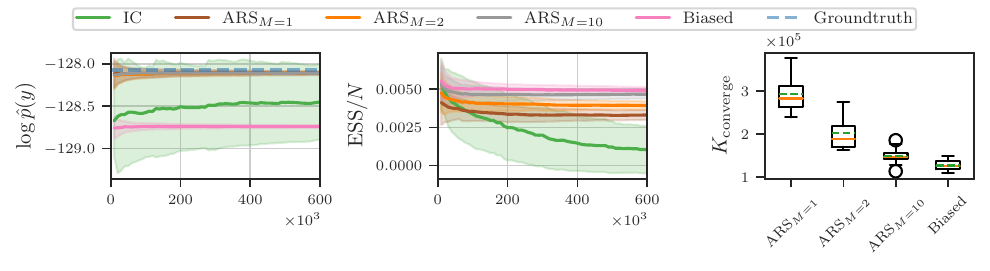}
        \includegraphics[trim=0 5 0 15,clip]{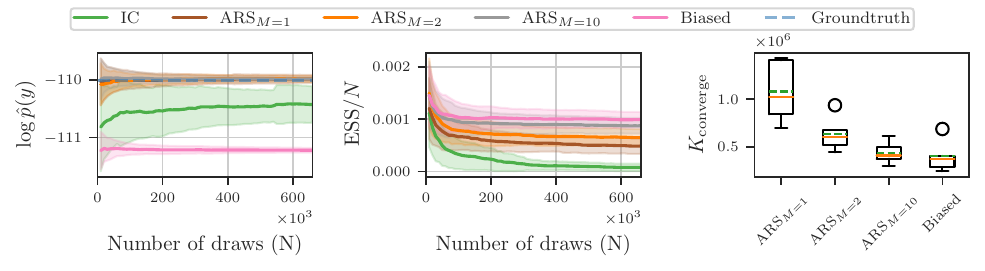}
        \caption{Experimental results for the Mini-SHERPA experiment with a channel 2 (top) and channel 3 (bottom) observation. For the convergence plot on the bottom row, since the methods are slow to converge, we choose $\epsilon = 0.03$.}
        \label{fig:app:more-results:mini-sherpa-results}
    \end{figure}

    \begin{minipage}[tbp]{\columnwidth}
        \begin{lstlisting}[caption=The Mini-SHERPA experiment, label=code:mini-sherpa]
def rejection_sample(scale):
    y_max = 1/scale
    while True:
        rs_start()
        x = sample(Uniform(-scale, scale))
        a = sample(Uniform(0, y_max))
        if a <= 1/scale*abs(x/scale):
            rs_end()
            return x
def Mini_SHERPA(obs):
    channel = sample(Categorical([1/3, 1/3, 1/3]))
    momentum_x = sample(Uniform(-0.5, 0.5))
    momentum_y = sample(Uniform(-0.5, 0.5))
    momentum_z = sample(Uniform(10, 20))
    thetas, phis = [], []
    for i in range(channel)
        thetas.append(rejection_sample(pi/4))
        phis.append(sample(Uniform(0, 2*pi)) for i in range(channel))
    deposits = simulate(momentum_x, momentum_y, momentum_z, thetas, phis)
    likelihood = Normal(deposits, max(sqrt(deposits, 0.3)))
    observe(likelihood, obs)
    return momentum_x, momentum_y, momentum_z, channel, deposits
        \end{lstlisting}
    \end{minipage}
In Program~\ref{code:mini-sherpa}, the function \texttt{simulate} computes the resulting particle momentums, simulates the resulting particles and renders the $20\times 20$ image of dispatched energies.

\section{MORE EXPERIMENTAL RESULTS}\label{app:more-results}

\subsection{Collapsed weighting}
In all experiments presented in this paper, the rejection sampling loops are simple enough to admit tractable ``collapsed'' distributions. The collapsed distribution is a Gaussian in the Maraglia experiment and a Beta in the Beta-Bernoulli experiment. In Mini-SHERPA, it is a custom distribution with tractable density and CDF function, therefore, we can implement a custom distribution for it.
One might argue that we can ask the user to implement the collapsed program instead of using rejection sampling and avoid the problems discussed in the paper. In this section we investigate how feasible it is and provide additional results to compare performance of collapsed weighting and ARS.

First, we remind the reader that ARS depends on the acceptance probability of the loops and not other features of them.
Therefore, ARS performs similarly in more complicated programs with intractable rejection loops.
Second, it is important to note that in many situations, the collapsed program is intractable.
For example, consider a generator of LaTeX source codes that compile without error. Therefore, it is essential for universal PPLs to provide the tooling to handle such cases without relying on the user to implement the collapsed program.

Nonetheless, in this section we provide additional results to compare performance of ARS compared to collapsed weighting.
However, re-implementing the program in its collapsed form changes the number of random variables and distribution types (for example, in the Marsaglia experiment, two Uniform random variables inside the rejection loop will be replaced by one Gaussian random variable). It will in turn reduce the variance of the estimator due to having fewer number of random variables and usually simpler distributions. It also requires re-training the inference network. Such differences improve performance of the collapsed model for reasons unrelated to the weighting.

For a fair comparison, we keep the model unchanged, but estimate or analytically compute (where applicable) the quantities of interest $p(z|x, A)$ and $z(z|x, A, y)$ and compute collapsed weighting according to its definition in \cref{eq:weight}. Results of this experiment are labelled ``Collapsed'' in \cref{fig:app-more-results:marsaglia,fig:app-more-results:beta}.

\begin{figure}[tbp]
    \centering
    \includegraphics{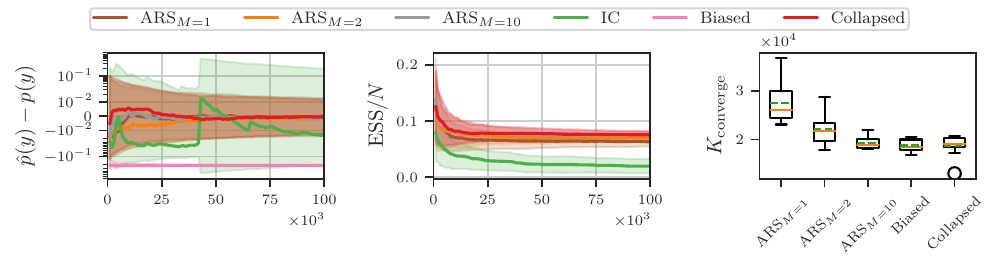}
    \includegraphics[trim=0 0 0 15,clip]{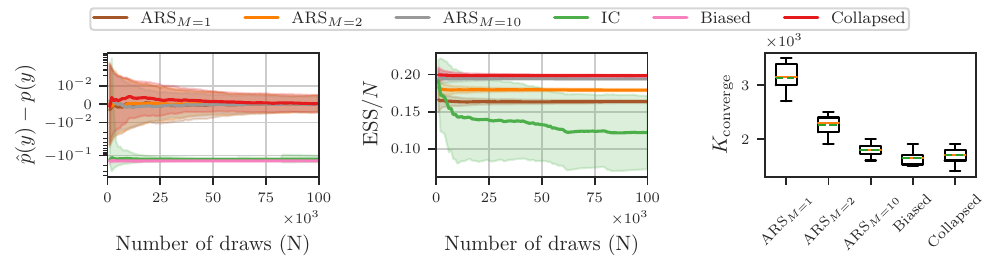}
    \caption{Additional results from the Marsaglia experiment. (Top) is the same as \cref{fig:results-marsaglia-mini_sherpa}, but includes the collapsed weighting as well. (Bottom) has a different observation. Observations are  ($y_1 = y_2 = 0$) and ($y_1 = y_2 = -1$) respectively for the top and bottom row.}
    \label{fig:app-more-results:marsaglia}
\end{figure}

\subsection{Additional results}
Here we provide additional experimental results for the three models considered in the main text, but the different observations. \cref{fig:app-more-results:marsaglia,fig:app:more-results:mini-sherpa-results,fig:app-more-results:beta} show additional results for the Marsaglia, Mini-SHERPA, and Beta models.

Since the lines in evidence and ESS plots are tightly clustered, we summarize the final mean and standard deviation of each line for each of Marsaglia and Mini-SHERPA experiments in \cref{fig:app-more-results:results-table}.

\begin{figure}[tbp]
    \centering
    \includegraphics[trim=0 130 0 0,clip]{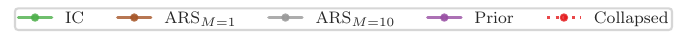}
    \includegraphics[trim=50 0 50 30,clip]{figs/exp_beta_a_collapsed.pdf}
    \includegraphics[trim=65 0 50 30,clip]{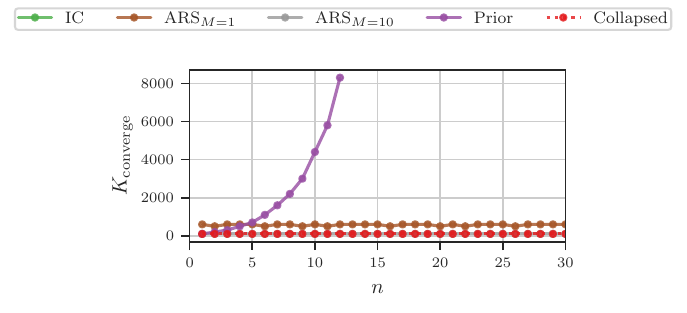}
    \caption{Additional results from the Beta experiment. (Left) is the same as \cref{fig:results_beta}, but includes the collapsed weighting. (Right) has a different set of observations and proposal distributions as explained in \cref{tab:app:beta-bern-proposals}.}
    \label{fig:app-more-results:beta}
\end{figure}

\begin{figure}[btp]
    \centering
    \begin{tabular}{c|c|c|c|c|c|c|c}
        \multicolumn{1}{c}{} & & IC                  & ARS,M=1              & ARS,M=2              & ARS,M=10 & Biased & Collapsed \\\hline
        \multirow{2}{*}{Marsaglia - a} & Bias ($\times 10^{-4}$) & $-151$ & $-1.3$  & $-16$  & $-14$ & $-2247$ & $-5$\\
        &$\mathrm{ESS} / N$ ($\times 10^{-3}$)             & $20$      & $63$       & $70$        & $75$ & $76$ & $76$\\\hline
        \multirow{2}{*}{Marsaglia - b} & Bias ($\times 10^{-4}$) & $-1285$ & $-0.5$  & $-0.5$  & $-1447$ & $2.3$\\
        &$\mathrm{ESS} / N$ ($\times 10^{-2}$)             & $12$      & $16$       & $18$        & $20$ & $20$\\
        \hline\hline 
        \multirow{2}{*}{Mini-SHERPA - 1} &Bias ($\times 10^{-4}$) & $-102$ & $-2.3$  & $-2.3$  & $-2.5$ & $85$ & -\\
        & $\mathrm{ESS} / N$ ($\times 10^{-2}$)              & $79$      & $83$       & $83$        & $84$ & $85$ & -\\\hline
        \multirow{2}{*}{Mini-SHERPA - 2} &Bias ($\times 10^{-2}$) & $-38$ & $-3.7$  & $-3.6$  & $-3.7$ & $67$ & -\\
        & $\mathrm{ESS} / N$ ($\times 10^{-3}$)             & $1.0$      & $3.3$       & $3.9$        & $4.7$ & $4.9$ & -\\\hline
        \multirow{2}{*}{Mini-SHERPA - 3} &Bias ($\times 10^{-2}$) & $-42$ & $2.6$  & $2.8$  & $2.6$ & $-120$ & -\\
        & $\mathrm{ESS} / N$ ($\times 10^{-4}$)              & $0.8$      & $5$       & $6$        & $9$ & $10$ & -\\
        \hline 
    \end{tabular}%
    \caption{Each row block shows estimation bias and average ESS of an experiment at the right-most point of their corresponding plot. ``Marsaglia - a'' corresponds to \cref{fig:app-more-results:marsaglia} (top) and \cref{fig:results-marsaglia-mini_sherpa} (top). ``Marsaglia - b'' corresponds to \cref{fig:app-more-results:marsaglia} (bottom). ``Mini-SHERPA - 1'' corresponds to \cref{fig:results-marsaglia-mini_sherpa} (bottom). ``Mini-SHERPA - 2'' and ``Mini-SHERPA - 3'' correspond to the top and bottom row of \cref{fig:app:more-results:mini-sherpa-results}}
    \label{fig:app-more-results:results-table}
\end{figure}

\end{document}